\documentclass[letterpaper, 10 pt, conference]{cls/ieeeconf}

\IEEEoverridecommandlockouts
\let\labelindent\relax

\usepackage[normalem]{ulem}	                        
\usepackage[table,usenames,dvipsnames]{xcolor}      
\usepackage{enumitem}                               
\usepackage{extarrows}                              
\usepackage[noadjust]{cite}                         

\usepackage{amsmath,amssymb,amsfonts,amsthm, dsfont} 		
\usepackage{algorithm,algorithmicx,listings}    
\usepackage[noend]{algpseudocode}			          
\usepackage{mathtools} 

\usepackage{graphicx,tabularx,subcaption}
\usepackage[export]{adjustbox}

\usepackage[font={small}]{caption}
\usepackage[titletoc,title]{appendix}

\usepackage[breaklinks=true, colorlinks, bookmarks=true, citecolor=Black, urlcolor=Violet,linkcolor=Black]{hyperref}

\newcommand{\prl}[1]{\left(#1\right)}
\newcommand{\brl}[1]{\left[#1\right]}
\newcommand{\crl}[1]{\left\{#1\right\}}

\newcommand{\norm}[1]{\left\lVert#1\right\rVert}

\newcommand*{\Bg}{\mathbf{g}}

\newcommand*{\Bq}{\mathbf{q}}
\newcommand*{\Br}{\mathbf{r}}
\newcommand*{\Bs}{\mathbf{s}}

\newcommand*{\Bu}{\mathbf{u}}
\newcommand*{\Bv}{\mathbf{v}}

\newcommand*{\Bx}{\mathbf{x}} 
\newcommand*{\By}{\mathbf{y}}
\newcommand*{\Bz}{\mathbf{z}}

\newcommand*{\BA}{\mathbf{A}} 
\newcommand*{\BB}{\mathbf{B}} 
\newcommand*{\BC}{\mathbf{C}}

\newcommand*{\BK}{\mathbf{K}} 
 
\newcommand*{\BI}{\mathbf{I}} 
\newcommand*{\BJ}{\mathbf{J}} 
 
\newcommand*{\BP}{\mathbf{P}} 
\newcommand*{\BQ}{\mathbf{Q}}

\newcommand*{\BU}{\mathbf{U}} 
\newcommand*{\BV}{\mathbf{V}} 

\newcommand*{\Bzero}{\mathbf{0}} 
 
\newcommand*{\Bxi}{\boldsymbol{\xi}}

\newcommand*{\C}[1]{\mathcal{#1}}

\newcommand*{\CE}{\mathcal{E}}
\newcommand*{\CF}{\mathcal{F}}

\newcommand*{\CP}{\mathcal{P}}
\newcommand*{\CO}{\mathcal{O}}

\newcommand*{\LS}{\mathcal{LS}}

\newcommand*{\Rn}{\mathbb{R}^n}

\newcommand*{\pdm}{\mathbb{S}^n_{>0}}
\newcommand*{\psdm}{\mathbb{S}^n_{\geq 0}}

\newcommand*{\rp}{\Bx}
\newcommand*{\rv}{\dot{\rp}}

\newcommand*{\gp}{\Bg}

\newcommand*{\damping}{\zeta}
\newcommand*{\drp}{\dot{\rp}}
\newcommand*{\ddrp}{\ddot{\rp}}

\newcommand*{\dgp}{\dot{\gp}}
\newcommand*{\lpg}{\bar{\gp}}

\newcommand*{\DeltaE}{\Delta E}

\newtheorem{theorem}{Theorem}

\newtheorem{lemma}{Lemma}
\theoremstyle{definition}
\newtheorem{definition}{Definition}

\newtheorem*{assumption*}{Assumption}
\newtheorem*{problem*}{Problem}
\newtheorem{problem}{Problem}
\theoremstyle{remark}

\newtheorem*{solution*}{Solution}

\def\thetitle{Fast and Safe Path-Following Control using a State-Dependent Directional Metric}
\def\theauthor{Zhichao Li, Omur Arslan, Nikolay Atanasov}
\def\thekeywords{quadratic norm, ellipsoidal trajectory bound, reference governor, Lyapunov theory}

\hypersetup{
  pdfauthor={\theauthor},%
  pdftitle={\thetitle},%
  pdfsubject={\thetitle},%
  pdfkeywords={\thekeywords}
}

\IEEEoverridecommandlockouts

\title{\LARGE \bf \thetitle}
\author{Zhichao Li$^{1}$ \and {\"O}m{\"u}r Arslan$^{2}$ \and Nikolay Atanasov$^{1}$%
\thanks{We gratefully acknowledge support from NSF CRII IIS-1755568, ARL DCIST CRA W911NF-17-2-0181, and ONR SAI N00014-18-1-2828.}
\thanks{$^{1}$Z. Li and N. Atanasov are with the Department of Electrical and Computer Engineering, University of California, San Diego, La Jolla, CA 92093, USA {\tt\small \{zhl355,natanasov\}@eng.ucsd.edu}.}%
\thanks{$^{2}${\"O}. Arslan is with the Department of Mechanical Engineering, Eindhoven University of Technology, P.O. Box 530, 5600 MB, Eindhoven, The Netherlands  {\tt\small o.arslan@tue.nl}.}%
}
\begin{document}
\maketitle
\begin{abstract}
This paper considers the problem of fast and safe autonomous navigation in partially known environments. Our main contribution is a control policy design based on ellipsoidal trajectory bounds obtained from a quadratic state-dependent distance metric. The ellipsoidal bounds are used to embed directional preference in the control design, leading to system behavior that is adapted to local environment geometry, carefully considering medial obstacles while paying less attention to lateral ones. We use a virtual reference governor system to adaptively follow a desired navigation path, slowing down when system safety may be violated and speeding up otherwise. The resulting controller is able to navigate complex environments faster than common Euclidean-norm and Lyapunov-function-based designs, while retaining stability and collision avoidance guarantees.
\end{abstract}

\section{Introduction}
\label{sec:introduction}
Advances in embedded sensing and computation have enabled robot applications in unstructured environments and in close interaction with humans, including autonomous transportation, inspection and cleaning services, and medical robotics.  Safe, yet efficient robot navigation is important for these applications but is challenging due to partially known or rapidly changing operational conditions. 

In motion planning, optimality guarantees have been achieved for geometric path planning~\cite{ARAstar,RRTstar} but incorporating robot dynamics without violating these guarantees remains an active area of research.  To achieve efficient behavior for a dynamical system, optimal control theory is used together with sampling-based or search-based motion planners. Sampling-based methods connect neighboring states using locally optimal control such as linear-quadratic-regulation (LQR)~\cite{lqr_rrt_star_perez2012} or fixed-final-state-free-final-time optimal control~\cite{kinodynamic_rrts_webb2013}. Locally optimal control, however, does not necessarily lead to global optimality~\cite{pacelli2018integration}. Search-based methods construct a safe corridor~\cite{SFC,SFC_FM} (a connected safe region in free space) and seek a composition of short motion primitives within. Depending on the primitive design, the resulting trajectory may already be dynamically feasible~\cite{Liu_DynamicTrajectoryPlanning_IROS17} or may be optimized locally using model predictive control (MPC)~\cite{SFC}. These techniques do not provide formal guarantees for joint collision avoidance and stability.

To guarantee safety formally, most existing works rely on Lyapunov theory and reachable set computations. A \emph{funnel} is an outer approximation of the reachable set of a dynamical system in the presence of disturbances~\cite{funnel_idea}. Building on the seminal work of~\cite{burridge1999sequential}, sequential composition of funnels offers effective means of guaranteeing safe navigation~\cite{lqr_tree_tedrake2009, Funnel_lib, gawron2017vfo, gawron_2018IROS_VFO}. Using sum-of-squares optimization~\cite{sos_stability}, these techniques can deal with nonlinear systems, nonholonomic constraints, and bounded disturbances. Recently, control barrier functions methods~\cite{CBF_ames2014rapidly, CBF_wu2015safety, CBF_ames2016control, CBF_quadrotor} have received significant attention. While optimizing performance without sacrificing stability using control Lyapunov functions (CLF), safety constraints are handled by a control barrier function (CBF). A virtual reference governor system~\cite{garone2016_ERG, kolmanovsky2014ref_cmd_gov} may also be used to enforce safety constraints as an add-on control scheme to pre-stabilized dynamical systems. Using a reference governor design, \cite{Gov_ICRA17} enables safe navigation for an acceleration-controlled robot among spherical obstacles by adaptive tracking a first-order vector field. 

\begin{figure}[t]
	\centering
	\includegraphics[width=\linewidth]{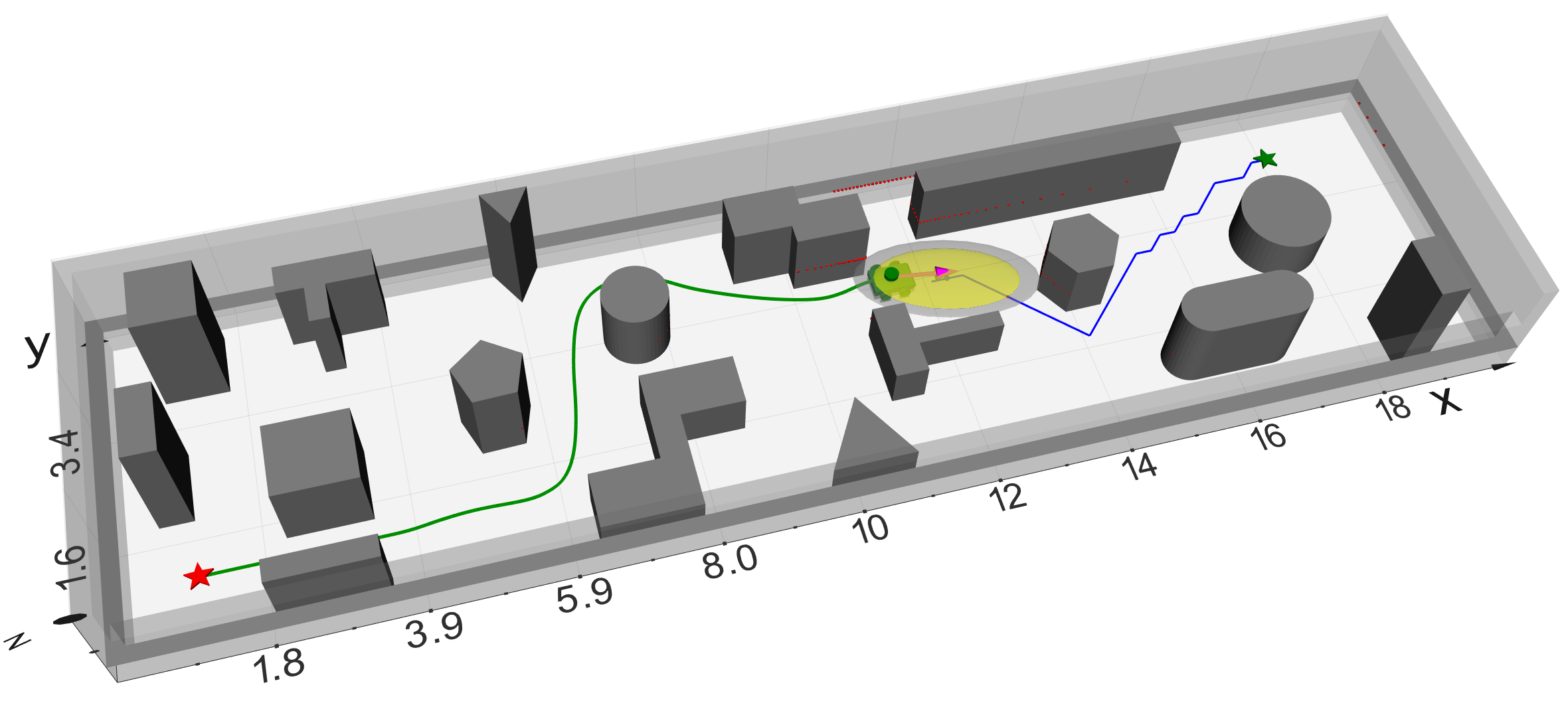}
	\caption{A robot equipped with a lidar scanner navigates in an unknown cluttered environment. A state-dependent metric that considers the robot's direction of motion is designed to approximate the robot's future trajectory (yellow ellipse) and quantify its safety (gray ellipse) with respect to surrounding obstacles. An adaptive controller guarantees safety and stability based on these measures.}
	\label{fig:complex_cooridor_sim}
	\vspace{-0.15in}
\end{figure}

The importance of considering configuration space geometry and system dynamics jointly when designing a steering function for sampling-based kinodynamic planning is discussed in~\cite{pacelli2018integration}. 
Metrics based on Mahalanobis distance, linear-quadratic-regulator cost, and a Gram matrix derived from system linearization are considered. Inspired by this work, we observe that using a static distance metric to quantify the safety of a robot with respect to surrounding obstacles can significantly impact its performance in real applications. For example, an autonomous golf cart running on campus has to simultaneously maintain safe distance from pedestrians and, yet, be able to squeeze through narrow passages such as doors or road block pillars. Static safety measures do not take the system's velocity direction into account, leading to overly cautious behavior even if the direction of travel is completely orthogonal to nearby obstacle surfaces. We refer to this limitation as the \emph{corridor effect} and aim to design an adaptive path-following controller, mitigating this effect via a new metric that takes the system's state into account when quantifying safety. 
The two main \textbf{contributions} of this work are highlighted as follows. First, we propose a new \emph{state-dependent directional metric} and develop accurate system trajectory bounds for linearized robot dynamics that take direction of motion into account. Second, we develop an adaptive feedback controller, based on the directional trajectory bounds, and prove that it ensures stable and collision-free navigation. The controller relies only on local obstacle information, easily obtainable from onboard sensors, and provides fast tracking performance in complex unknown environments. See Fig.~\ref{fig:complex_cooridor_sim} for an illustration.
\section{Notation}
Let $\pdm$ and $\psdm$ denote the set of $n\times n$ symmetric positive and semi-definite matrices. Let $\succ$ and $\succeq$ denote the generalized inequalities associated with $\pdm$ and $\psdm$. Denote the Euclidean ($\ell^2$) norm by $\norm{\Bx}$ and the quadratic norm induced by $\BQ \in \pdm$ as $\norm{\Bx}_\BQ \coloneqq \sqrt{ \Bx^T \BQ \Bx}$. Let $\lambda_{\max}(\BQ)$ and $\lambda_{\min}(\BQ)$ be the maximum and minimum eigenvalues of $\BQ$. Let $d_\BQ(\Bx, \C{A}) \coloneqq \inf_{\mathbf{a} \in \C{A}} \norm{\Bx-\mathbf{a}}_\BQ$ denote the quadratic norm distance from a point $\Bx$ to a set $\C{A}$. Given $\BQ \in \pdm$ and scaling $\eta \geq 0$, denote the associated ellipsoid centered at $\Bq \in \Rn$ by $\C{E}_\BQ(\Bq , \eta) \coloneqq \crl{\Bx \in \Rn \mid (\Bx-\Bq )^T \BQ (\Bx-\Bq ) \leq \eta}$.

\section{Problem Formulation}
\label{sec:problem_formulation}
Consider a robot operating in an unknown environment $\C{W} \subset \mathbb{R}^n$. Denote the obstacle space by a closed set $\CO$ and the free space by an open set $\C{F} \coloneqq \C{W}  \setminus \C{O}$.
Suppose that the robot dynamics are controllable, linear, and time-invariant:
\begin{equation}
\label{eq:lin_system}
\dot{\Bs} = \BA \Bs + \BB \Bu
\end{equation}
where $\Bu \in \mathbb{R}^{n_u}$ is the control input and $\Bs := \prl{\Bx,\By} \in \mathbb{R}^{n_s}$ is the robot state\footnote{Transpose operations are omitted when grouping vectors for conciseness.}, decomposed into constrained variables $\Bx$ and free variables $\By$. Throughout this paper, $\Bx(t) \in \Rn$  represents the robot position, required to remain within $\CF$ for all $t \geq 0$, while $\By(t) \in \mathbb{R}^{n_s - n}$ denotes higher-order (velocity, acceleration, jerk, \ldots) terms. Our objective is to design a closed-loop control policy $\Bu(t)$, ensuring that the robot follows a given navigation path in the free space.
\begin{definition}
A \emph{path} is a piecewise-continuous function $\CP: \brl{0,1} \mapsto \mathring{\CF}$ that maps a path-length parameter $\alpha \in \brl{0,1}$ to the interior of free space. The start $\CP(0)$ and the end $\CP(1)$ of a navigation path $\CP(\alpha)$ are in the interior of free space, i.e., $\CP(0), \CP(1) \in \mathring{\CF}$.
\end{definition}
A path $\CP$ may be provided by a geometric planning algorithm~\cite{lavalle2006planning,ARAstar,RRTstar}. We consider the following problem.
\begin{problem}
Given a path $\CP$, design a control policy $\Bu(t)$ so that the robot~\eqref{eq:lin_system} is asymptotically steered from the start to the end of $\CP$ while remaining collision-free, i.e., $\rp(t) \in \CF$ for all $t \geq 0$. 
\end{problem}

\section{Technical Approach}
\label{sec:technical_approach}
In this section, we propose a novel state-dependent directional metric (SDDM) and show that closed-loop trajectories of~\eqref{eq:lin_system} can be bounded in the SDDM by solving a convex optimization problem. We develop a feedback control law that exploits the trajectory bounds to stabilize the robot and follow the path $\Br$ adaptively, slowing down when safety may be endangered and speeding up otherwise.

\subsection{State-Dependent Directional Metric}
\label{sec:directional_sense}
As mentioned in the introduction, measuring safety using a static Euclidean norm may lead to system performance suffering from the \emph{corridor effect}. We propose a quadratic distance measure $\norm{\cdot}_\BQ$ that assigns priority to obstacles depending on the robot's moving direction. The level sets of $\norm{\cdot}_\BQ$ are ellipsoids $\C{E}_\BQ(\Bzero, \eta)$ whose shape and orientation are determined by the matrix $\BQ$. Our idea is to encode a desired directional preference in the distance metric via an appropriate choice of $\BQ$. Consider the example in Fig.~\ref{fig:norm_comparison}. A quadratic norm, well-aligned with the local environment geometry, may provide a more accurate evaluation of safety than a static Euclidean norm. Based on this observation, we propose a general construction of a directional matrix $\BQ[\Bv]$, in the direction of vector $\Bv$, that defines a state-dependent directional metric.
\begin{definition}
	\label{def:directional_matrix}
	A \emph{directional matrix} associated with vector $\Bv$ and scalars $c_2 > c_1 > 0$ is defined as
	\begin{equation} \label{eq:directional_matrix}
	\BQ \brl{\Bv} =\begin{cases}
	c_2 \BI + (c_1 - c_2) \frac{\Bv \Bv^T}{\norm{\Bv}^2}, & \text{if}\ \Bv \neq 0, \\
	c_1 \BI, & \text{otherwise}.
	\end{cases}
	\end{equation} 
\end{definition}
The unit ellipsoid $\CE_{\BQ\brl{\Bv}}(\Bx, 1)$ centered at $\Bx$ generated by a directional matrix $\BQ\brl{\Bv}$ is elongated in the direction of $\Bv$.

\begin{figure}[t]
	\centering
	\includegraphics[width=0.48\linewidth,valign=t]{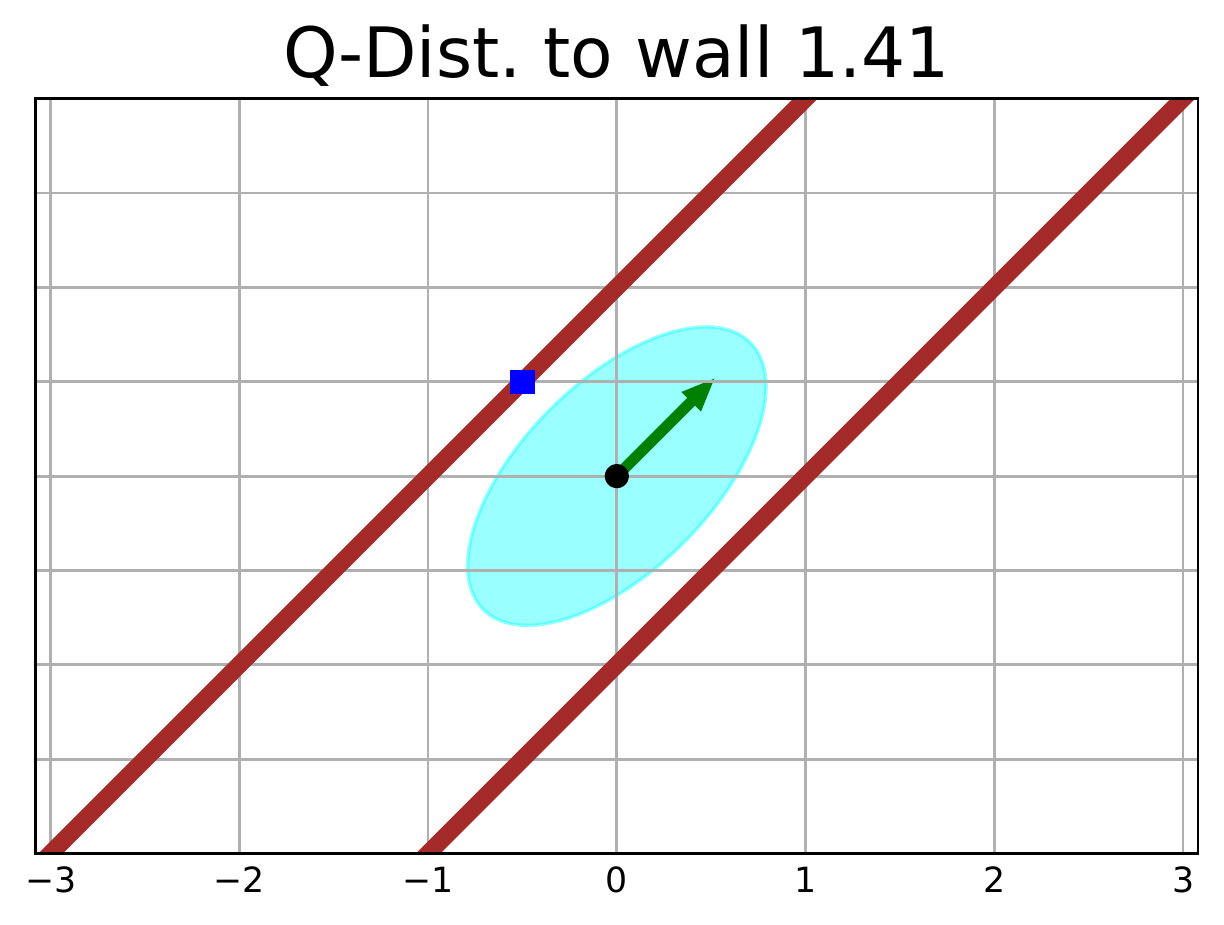}
	\includegraphics[width=0.48\linewidth,valign=t]{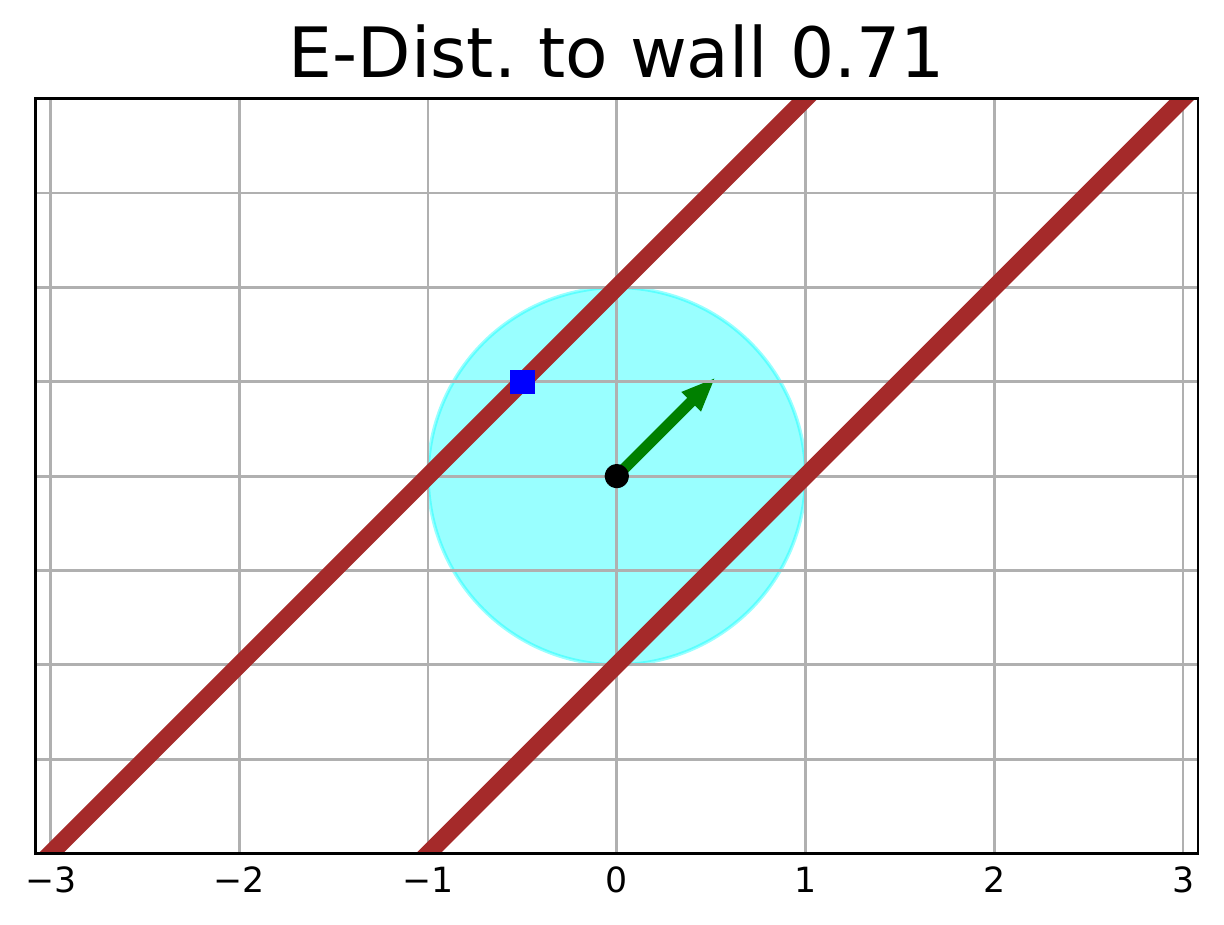}
	\caption{Robot (black dot) moving in direction $\Bv := \brl{\sqrt{2}/2, \sqrt{2}/2}$ (green arrow) along a corridor. The distances, measured by a quadratic norm $\norm{\cdot}_\BQ$ (left) and Euclidean norm $\norm{\cdot}$ (right), from the robot to the closest point (small blue square) on the wall (red line) are 1.41 and 0.71. The matrix $\BQ = [[2.5 -1.5],[-1.5, 2]]$ is defined as a \emph{directional matrix} $\BQ\brl{\Bv}$.}
	\label{fig:norm_comparison}
\end{figure}

\begin{lemma}
\label{lem:directional_metric}
For any vector $\Bv$, the directional matrix $\BQ \brl{\Bv}$ is symmetric positive definite.
\end{lemma}
\begin{proof}
Since $\Bv \Bv^T$ is symmetric, $\BQ\brl{\Bv}^T = \BQ\brl{\Bv}$. If $\Bv=\Bzero$, $\BQ\brl{\Bv} = c_1 \BI$ is positive definite. If $\Bv \neq \Bzero$ and $\Bq$ is arbitrary:
\begin{align*}
\Bq^T \BQ\brl{\Bv} \Bq  &= c_2 \Bq^T \Bq + (c_1 - c_2) \frac{(\Bq^T \Bv)^2}{\norm{\Bv}^2} \\
            &\geq c_2 \Bq^T \Bq + (c_1 - c_2) \frac{\norm{\Bq}^2 \norm{\Bv}^2}{\norm{\Bv}^2} = c_1 \norm{\Bq}^2,
\end{align*}
which follows from $c_2 > c_1$ and the Cauchy-Schwarz inequality. The proof is completed by noting that $c_1 > 0$.
\end{proof}
\subsection{Trajectory Bounds using SDDM}
\label{sec:ellipsoid_est}
Using a directional matrix, one can define an SDDM to adaptively evaluate the risk of surrounding obstacles. We will show how to use an SDDM to obtain bounds on the closed-loop trajectory of the constrained state $\Bx(t)$ in~\eqref{eq:lin_system}. Assume the robot is stabilized by a feedback controller $\Bu = -\BK \Bs$. The closed-loop dynamics are:
\begin{equation}
	\label{eq:LTI_closed_loop}
	\dot{\Bs} 	= \bar{\BA} \Bs \qquad
	\Bz 		= \BC \Bs 
\end{equation}
where $\bar{\BA} := (\BA - \BB \BK)$ is Hurwitz. Any initial state $\Bs_0 \coloneqq \Bs(t_0)$ will converge exponentially to the equilibrium point at origin. An output $\Bz$ is introduced to consider the constrained state $\Bx$. We are interested in measuring the maximum deviation of $\Bx(t)$ for $t \geq 0$ from the origin using a directional measure determined by the orientation of initial state $\Bx_0 \coloneqq \Bx(t_0)$ with respect to $\mathbf{0}$. Define an SDDM using the directional matrix:
\begin{equation}
\label{eq:directional_matrix_by_IC}
	 \BQ	\coloneqq \BQ \brl{\Bzero -\Bx_0} \in \pdm
\end{equation}
and choose output $\Bz(t) =  \BQ^{\frac{1}{2}} \Bx(t)$ so that $\BC \coloneqq \BQ^{\frac{1}{2}} \BP$, where $\BP \coloneqq \brl{\BI, \Bzero}$ is the projection matrix from $\Bs$ to $\Bx$. Note that $\Bz(t)^T \Bz(t) = \Bx(t)^T \BQ \Bx(t) = \norm{\Bx(t)}^2_{\BQ}$. Thus, measuring the maximum deviation of $\Bx(t)$ in the SDDM is equivalent to finding the output peak along the robot trajectory.
\begin{equation}
\label{eq:eta_def}
\eta(t_0) \coloneqq \max_{t\geq t_0} \norm{\rp(t)}_{\BQ}^2 = \max_{t\geq t_0}  \|\Bz(t)\|^2 
\end{equation} 
We outline two approaches to solve this problem.

\subsubsection{Exact solution} 
\label{sec:traj_est_exact_bound}
The output peak $\eta(t_0)$ can be computed exactly by comparing the values of $\|\Bz(t)\|^2$ at the boundary point $t = t_0$ and all critical points $\crl{t > t_0 \mid \frac{d}{dt} \|\Bz(t)\|^2 = 0}$. Since the closed-loop system in~\eqref{eq:LTI_closed_loop} is linear time-invariant, $\Bs(t)$ can be obtained in closed form. Let $\bar{\BA} = \BV \BJ \BV^{-1}$ be the Jordan decomposition of $\bar{\BA}$, where $\BJ$ is block diagonal. The critical points satisfy:
\begin{align}
0 &= \frac{d}{dt} \Bz(t)^T\Bz(t) = 2\Bz(t)^T \dot{\Bz}(t) \label{eq:critical_points}\\
&= 2\prl{\BP \BV e^{\BJ (t-t_0)} \BV^{-1} \Bs_0}^T \!\!\BQ \prl{\BP \BV \BJ e^{\BJ (t-t_0)} \BV^{-1} \Bs_0}.\notag
\end{align}
In general, an exact solution may be hard to compute due to the complicated expression of $e^{\BJ t}$.
\subsubsection{Approximate solution} 
\label{sec:traj_est_upper_bound}
When an exact solution to~\eqref{eq:eta_def} is hard to obtain, we may instead compute a tight upper bound on $\eta(t_0)$. Given a $\BU \in \mathbb{S}^{n_s}_{>0}$,
\begin{equation}
\label{eq:invariant_ellipsoid}
\CE_{inv}  \coloneqq \crl{\Bxi \in \mathbb{R}^{n_s} \mid \Bxi^T \BU \Bxi \leq 1} 
\end{equation} 
is an invariant ellipsoid for the robot dynamics~\eqref{eq:LTI_closed_loop}, i.e., $\Bs(t) \in \CE_{inv}$ for all $t \geq t_0$. Instead of finding the peak value of $\|\Bz(t)\|^2$ along the state trajectory, we can compute it over the invariant ellipsoid $\CE_{inv}$. Since $\CE_{inv}$ contains the system trajectory, we have for all $t \geq t_0$:
\begin{equation}
\label{eq:relax_relation}
\norm{\Bz(t)}^2 \leq \eta(t_0) \leq \max_{\Bxi \in \CE_{inv}} \Bxi^T \BC^T \BC \Bxi  
\end{equation}
Obtaining the upper bound above is equivalent to solving the following semi-definite program~\cite[Ch.6]{boyd_LMI_book}:
\begin{equation}
\begin{aligned} \label{eq:SDP_UB}
& \underset{\BU, \delta}{\text{minimize}} & &  \delta  \\
& \text{subject to}
& &  \bar{\BA}^T \BU + \BU\bar{\BA} \preceq \Bzero, \;\; \Bs_0^T \BU \Bs_0 \leq 1 \\
&  & & \begin{bmatrix}
\BU 	& \BC^T \\
\BC 	& \delta \BI
\end{bmatrix} \succeq \Bzero ,\quad\;\;   \BU \succ \Bzero. 
\end{aligned}
\end{equation}

\begin{lemma}
\label{lemma:traj_bound}
For any initial condition $\Bs_0$ and associated constant directional matrix $\BQ$ in \eqref{eq:directional_matrix_by_IC}, the trajectory $\rp(t)$ under system dynamics~\eqref{eq:LTI_closed_loop} admits a tight ellipsoid bound, $\rp(t) \in \CE_\BQ(\Bzero, \eta(t_0)) \subseteq \CE_\BQ(\Bzero, \delta(t_0))$, for all $t \geq t_0$, where $\eta(t_0)$ is the solution to~\eqref{eq:eta_def} and $\delta(t_0)$ is the solution to~\eqref{eq:SDP_UB}.
\end{lemma}
\begin{proof}
By definition, $\Bx(t) \in \CE_\BQ(\Bzero, \eta(t_0))$ is equivalent to $d^2_\BQ \prl{\Bzero, \Bx(t)} \leq  \eta(t_0)$. Since $\delta(t_0) =  \max_{\Bxi \in \CE_{inv}}\Bxi^T \BC^T \BC \Bxi$, inequality~\eqref{eq:relax_relation} yields
$\delta(t_0) \geq  \eta(t_0) \geq \norm{\Bz(t)}^2 = \norm{\Bx(t)}_\BQ^2 = d^2_\BQ \prl{\Bzero, \Bx(t)}$. Hence, $\Bx(t)\!\in\! \CE_\BQ(\Bzero, \eta(t_0)) \!\subseteq\!\CE_\BQ(\Bzero, \delta(t_0))$.\qedhere
\end{proof}
Now, we know how to find an accurate outer approximation of the system trajectory in the SDDM defined by~\eqref{eq:directional_matrix_by_IC}. We are ready to develop a feedback controller that utilizes the trajectory bounds to quantify the safety of the system with respect to surrounding obstacles, while following the navigation path towards the goal.
\subsection{Structure of the Robot-Governor Controller}
\label{sec:rgs_intro}
The problem of collision checking is simple for first-order kinematic systems since they can stop instantaneously to avoid collisions. We introduce a \emph{reference governor}~\cite{garone2016_ERG, kolmanovsky2014ref_cmd_gov}, a virtual first-order system:
$\dgp = \Bu_\gp$ with state $\gp \in \Rn$ and control input $\Bu_\gp \in \Rn$, which will serve to simplify the conditions for maintaining stability and safety concurrently. Our proposed structure of a path-following control design is shown in Fig.~\ref{fig:structure_safe_tracking}. The reference governor behaves as a real-time reactive trajectory generator that continuously regulates a reference signal for the real robot dynamics depending on risk level evaluation using SDDM. More precisely, we choose the real system's control input $\Bu$ so that the robot tracks the governor state $\gp$, while $\gp$ is regulated via $\Bu_\gp$ to ensure collision avoidance and stability for the joint robot-governor system.

\begin{figure}[t]
	\centering
	\includegraphics[width=1\linewidth,valign=t]{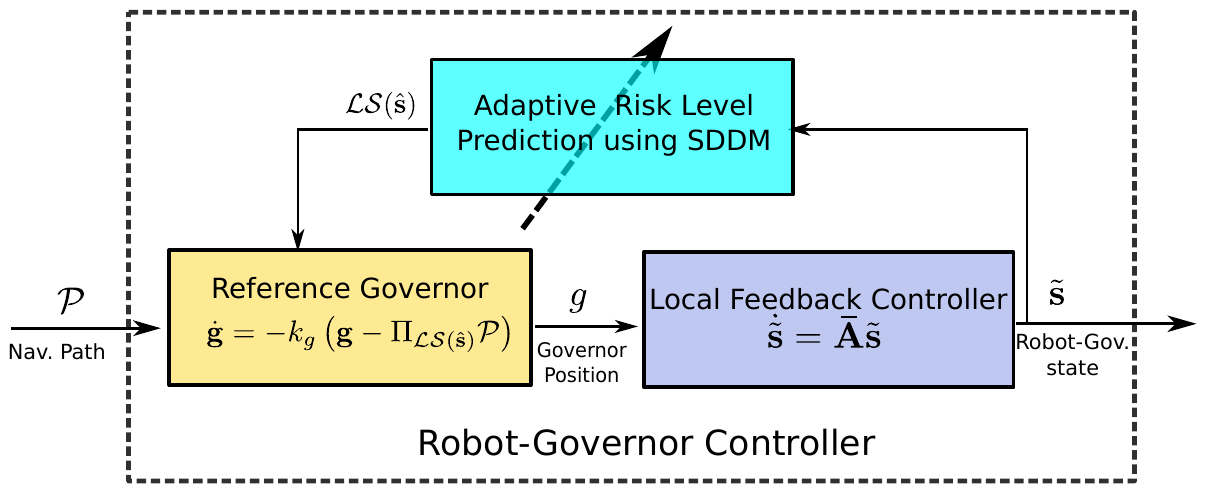}
	\caption{Structure of the proposed controller. A virtual reference governor adaptively conveys global navigation information to a local feedback controller based on the prediction of robot position trajectory. A local safe zone $\LS(\hat{\Bs})$, depending only on the current system state $\hat{\Bs}$, is constructed from an SDDM trajectory bound. A time-varying local goal $\Pi_{\LS(\hat{\Bs})} \Br$ is obtained by projecting the navigation path $\Br$ onto the local safe zone $\LS(\hat{\Bs})$. The governor chases the local goal and continuously sends its updated state $\gp$ as a reference signal to guide the local controller.
	}
	\label{fig:structure_safe_tracking}
\end{figure}

In detail, let $\tilde{\Bs} \coloneqq \Bs - \BP^T\Bg$ be the system state with the first element changed from $\Bx$ to $(\Bx - \gp)$ to make $\prl{\gp,\mathbf{0}}$ an equilibrium point. Choose a local controller for~\eqref{eq:lin_system} that tracks the governor state $\gp$:
\begin{equation}
\label{eq:local_feedback_controller}
\Bu = - \BK\, \tilde{\Bs}.
\end{equation}
Consider the augmented \emph{robot-governor system} with state $\hat{\Bs} = \prl{\tilde{\Bs}, \gp} \in \mathbb{R}^{(n_s + n)}$, coupling the real states with the governor state:
\begin{equation}
\label{eq:rgs_sum}
\dot{\hat{\Bs}}  = 
\begin{bmatrix}
\dot{\tilde{\Bs}} \\ \dgp
\end{bmatrix} = 
\begin{bmatrix}
	\bar{\BA} \tilde{\Bs} \\ \Bu_\Bg
\end{bmatrix}.
\end{equation}
Before proposing the design of the governor controller $\Bu_\gp(t)$, we analyze the behavior of the robot-governor system in the case of static governor.
\begin{lemma}
\label{lem:lyap_E}
If the governor is static, i.e., $\Bu_\gp(t)\equiv 0$ so that $\gp(t) \equiv \gp_0$, then the robot-governor system in~\eqref{eq:rgs_sum} is globally exponentially stable with respect to the equilibrium $\prl{\gp_0, \Bzero, \gp_0}$.
\end{lemma}
\begin{proof}
The subsystem $\dot{\tilde{\Bs}} = \bar{\BA} \tilde{\Bs}$ has an equilibrium at $\prl{\gp_0, \Bzero}$, which is globally exponentially stable because $\bar{\BA}$ is Hurwitz, while $\gp(t) \equiv \gp_0$ by assumption.
\end{proof}

In addition to guaranteeing stability for a static governor, we can use the ellipsoidal trajectory bounds from Lemma~\ref{lemma:traj_bound} to ensure safety.

\begin{theorem}
\label{thm:safety_static_gov}
Let $(\Bx_0 - \gp_0, \By_0, \gp_0)$ be any initial state for the robot-governor system in~\eqref{eq:rgs_sum} with $\Bx_0, \gp_0 \in \CF$. Suppose that $\Bu_\gp(t)\equiv \Bzero$ so that $\gp(t) \equiv \gp_0$. Let $\BQ \coloneqq \BQ \brl{\gp_0 - \Bx_0} \in \pdm$ be a constant directional matrix and suppose that the following safety condition is satisfied:
\begin{equation}
\label{eq:safety_condition}
\delta(t_0) \leq d_{\BQ}^2(\gp_0, \CO),
\end{equation}
where $\delta(t_0)$ is an upper bound for $\|\rp(t)-\gp_0\|_{\BQ}^2$ obtained according to Lemma \ref{lemma:traj_bound}. Then, the robot-governor system is globally exponentially stable with respect to the equilibrium $(\gp_0, \Bzero, \gp_0)$ and, moreover, the robot trajectory is collision free, i.e., $\rp(t) \in \CF$, for all $t \geq t_0$.
\end{theorem}

\begin{proof}
Since $(\Bg_0,\Bzero)$ is an equilibrium for $\dot{\tilde{\Bs}} = \bar{\BA} \tilde{\Bs}$, by Lemma~\ref{lemma:traj_bound}, $\rp(t) \in \CE_\BQ \prl{\gp_0, \delta(t_0)}$ for all $t \geq t_0$. From the safety condition in~\eqref{eq:safety_condition}, $\rp(t) \in \CE_\BQ \prl{\gp_0, \delta(t_0)} \subseteq \CE_\BQ(\gp_0, d^2_{\BQ} (\gp_0, \CO)) \subseteq \CF$ for all $t\geq t_0$. Stability is ensured by Lemma~\ref{lem:lyap_E}.
\end{proof}
\subsection{Local Projected Goal and Governor Control Policy}
\label{sec:rgs_lpg_ug}
We established that the robot-governor system is stable and safe as long as the governor is static and \eqref{eq:safety_condition} holds. Next, we consider how to move the governor without violating these properties. Based on~\eqref{eq:local_feedback_controller}, we know that the robot will attempt to track the governor state. The main idea is to choose a time varying directional matrix $\BQ(t)  = \BQ \brl{\gp(t) - \rp(t)}$ to measure the system's safety. Since $\BQ(t)$ is positive definite (by Lemma~\ref{lem:directional_metric}), it can still be used to define an SDDM $\|\cdot\|_{\BQ(t)}$. Then, Lemma~\ref{lemma:traj_bound}, can still provide an accurate robot trajectory bound $\CE_{\BQ(t)} \prl{\gp(t), \delta(t)}$, which takes the robot's direction of motion into account. We need to design the governor control policy $\Bu_\gp(t)$ so that $\delta(t)$ never violates a time-varying version of the safety condition in~\eqref{eq:safety_condition}.

Our approach is to define an ellipsoid $\LS(\hat{\Bs})$, called a \emph{local safe zone}, centered at the governor state $\gp$, and have the size of $\LS(\hat{\Bs})$ determine how fast the governor can move. In the worst case, if system safety or stability are endangered, $\LS(\hat{\Bs})$ should shrink to a point, forcing the governor to remain static. Once there is enough leeway in the safety conditions in~\eqref{eq:safety_condition}, $\LS(\hat{\Bs})$ can grow, allowing the governor to move without endangering the safety or stability. 

\begin{definition}
\label{def:LS}
A \emph{local safe zone} is a time-varying set that at time $t$ depends on the robot-governor state $\hat{\Bs}(t)$ as follows:
\begin{equation}
\label{eq:LS}
\LS(\hat{\Bs}) \coloneqq \crl{ \Bq \in \CF \mid d_{\BQ}^2(\Bq, \gp) \leq  \max \prl{0, \DeltaE(\hat{\Bs})} },
\end{equation}
where $\BQ(t) \coloneqq \BQ \brl{\gp(t)-\rp(t)}$ is a directional matrix, determined by $\hat{\Bs}(t)$, and $\DeltaE(\hat{\Bs}(t)) \coloneqq d_{\BQ(t)}^2(\gp(t), \mathcal{O}) - \delta(t)$ is a measure of leeway to safety violation, determined by an upper bound $\delta(t)$ on $\|\Bx(\tau)-\Bg(\tau)\|_{\BQ(t)}$ for all $\tau \geq t$ and the directional distance $d_{\BQ(t)}^2(\gp(t), \mathcal{O})$ from the governor to the nearest obstacle.
\end{definition}
The term $\DeltaE(\hat{\Bs})$ estimates safety of the system based on local environment geometry and robot activeness. The requirement $\DeltaE(\hat{\Bs}) \geq 0$ only places a constraint on the magnitude $\norm{\dgp}$, so $\dot{\gp}/ \norm{\dgp}$ is a degree of freedom that can be utilized to make $\gp$ asymptotically tend to a desired goal. We define a local goal for the governor.

\begin{definition}
\label{def:LPG}
A \emph{local projected goal} is the farthest point along the path $\Br$ contained in the local safe zone $\LS(\hat{\Bs})$:
\begin{equation}
\label{eq:local_projected_goal}
\alpha^* \coloneqq \max_{\alpha \in [0,1]} \crl{ \alpha \mid  \Br(\alpha) \in \LS(\hat{\Bs})} \quad \lpg(\hat{\Bs}) = \Br(\alpha^*)
\end{equation}
\end{definition}
The informal notation $\lpg(\hat{\Bs}) = \Pi_{\LS(\hat{\Bs})} \Br$ will be used for the local projected goal to emphasize that $\lpg(\hat{\Bs})$ is determined by projecting the path $\Br$ onto the local safe zone $\LS(\hat{\Bs})$. 
The structure of the complete closed-loop control policy is illustrated in Fig.~\ref{fig:structure_safe_tracking}, while the definitions of a local safe zone and a local projected goal are visualized in Fig.~\ref{fig:corridor_snapshot}.
We are finally ready to define the \emph{governor control policy}:
\begin{equation}
\label{eq:move2proj_goal}
\Bu_\gp \coloneqq -k_g \prl{ \gp - \lpg(\hat{\Bs}) } 
\end{equation}
where $k_g > 0$ is a control gain for the governor controller. We prove that the closed-loop system
is stable, safe, and asymptotically reaches the goal specified by the path $\Br$. We also informally claim that the path-following controller is fast due to the use of directional information for safety verification. This claim is supported empirically in Sec.~\ref{sec:evaluation}.

\begin{theorem}
\label{thm:safe_following}
Given a path $\Br$, the closed-loop robot-governor system~\eqref{eq:rgs_sum} with move-to-projected-goal~\eqref{eq:move2proj_goal} is asymptotically steered from any safe initial state $\hat{\Bs}_0$, i.e., $\DeltaE(\hat{\Bs}_0) > 0$ and $\Br(\alpha) \in \LS(\hat{\Bs}_0)$ for some $\alpha \in [0,1]$, to the goal state $\hat{\Bs}^* = (\Br(1), \Bzero, \Br(1))$ and the robot trajectory is collision-free for all time, i.e., $\rp(t) \in \CF$ for all $t \geq 0$.
\end{theorem}
\begin{proof}
From Thm.~\ref{thm:safety_static_gov}, we know that the robot-governor system~\eqref{eq:rgs_sum} will asymptotically converge to the equilibrium point $(\gp, 0, \gp)$ without collisions if the governor is static. The governor control policy in~\eqref{eq:move2proj_goal} allows the governor to move only when the interior of $\LS(\hat{\Bs})$ is nonempty. From Def.~\ref{def:LS}, this happens only if the safety condition is strictly satisfied, i.e., $\DeltaE(\hat{\Bs}) = d_{\BQ(t)}^2(\gp(t), \mathcal{O}) -  \delta(t) > 0$. Since $\rp$ approaches $\gp$, $\DeltaE(\hat{\Bs})$ eventually becomes strictly positive and the set $\LS(\hat{\Bs})$ becomes an ellipsoid in free space with non-empty interior. Since initially $\Br(\alpha) \in \LS(\hat{\Bs}_0)$ for some $\alpha \in [0,1]$, the local projected goal in~\eqref{eq:local_projected_goal} will be well defined and when $\LS(\hat{\Bs})$ grows, the projected goal will move further along the path $\Br$, i.e., the path length parameter $\alpha$ will increase. Since the system dynamics are continuous, $\DeltaE(\hat{\Bs})$ cannot suddenly become negative without crossing zero.  If $\DeltaE(\hat{\Bs}) \downarrow 0$, the local energy zone $\LS(\hat{\Bs})$ shrinks to a point, i.e., $\LS(\hat{\Bs}) = \crl{\gp}$, and hence the governor stops moving and waits until the robot catches up. When the governor is static, and since the safety condition in~\eqref{eq:safety_condition} is satisfied, Thm.~\ref{thm:safety_static_gov} again guarantees that the robot can approach the governor without collisions, increasing $\DeltaE(\hat{\Bs})$ in the process. Once $\DeltaE(\hat{\Bs})$ goes above $0$, the governor starts moving towards the goal again by chasing the projected goal. Note that the local projected goal always lies on the navigation path inside the free space, i.e., $\lpg \in \Br \subset \mathring{\CF}$, so the robot-governor system cannot remain stuck at any configuration except $\hat{\Bs}^* = (\Br(1), \Bzero, \Br(1))$. Using LaSalle's Invariance Principle~\cite{khalil2002nonlinear}, we can conclude that the largest invariant set is the point where both the robot and the governor are stationary at the goal location $\Br(1)$.
\end{proof}

\begin{figure}[t]
\centering
\includegraphics[width=0.45\linewidth]{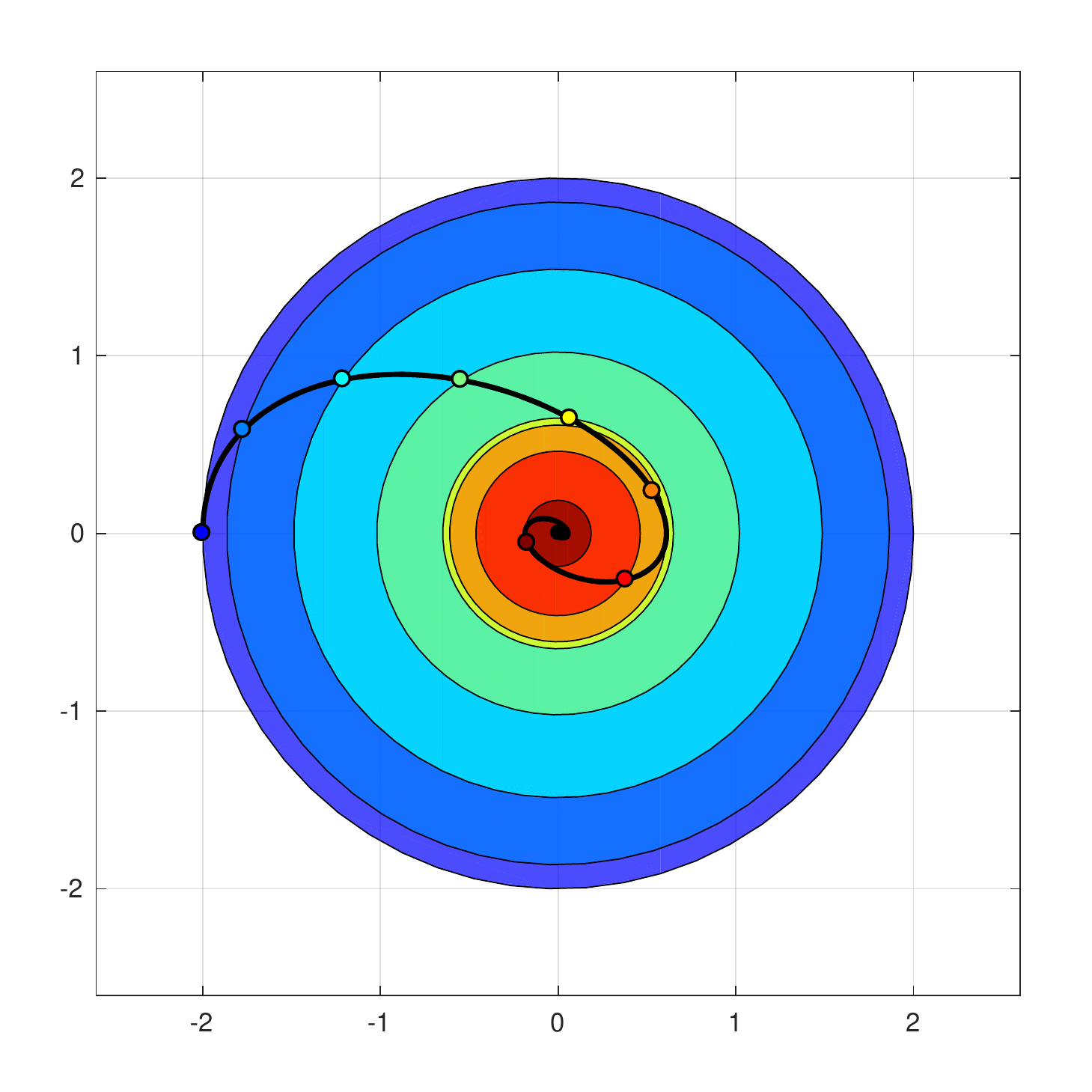}%
\hfill %
\includegraphics[width=0.22\textwidth]{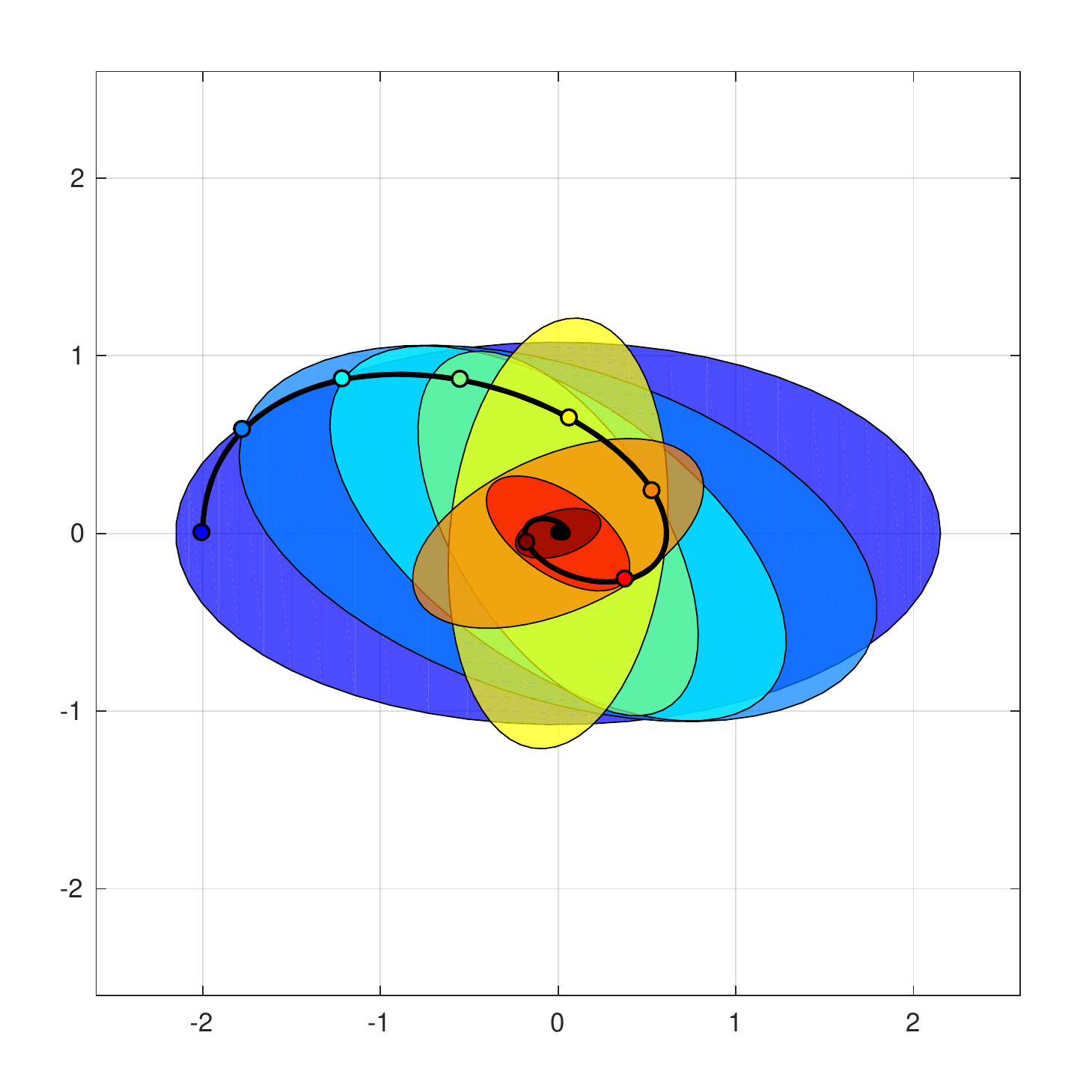}%
\caption{Trajectory bounds comparison between a Euclidean metric (left) and an SDDM (right). The governor is fixed at the origin while the robot's initial conditions are $\rp_0 = (-2, 0)$ and $\rv_0 = (0, 2)$. The change of the trajectory bounds over time is illustrated via ellipsoids with different colors, starting from cold/blue and converging towards warm/red.}
\label{fig:traj_est_comp}
\end{figure}

\begin{figure}[t]
	\centering
	\includegraphics[width=1\linewidth,valign=t]{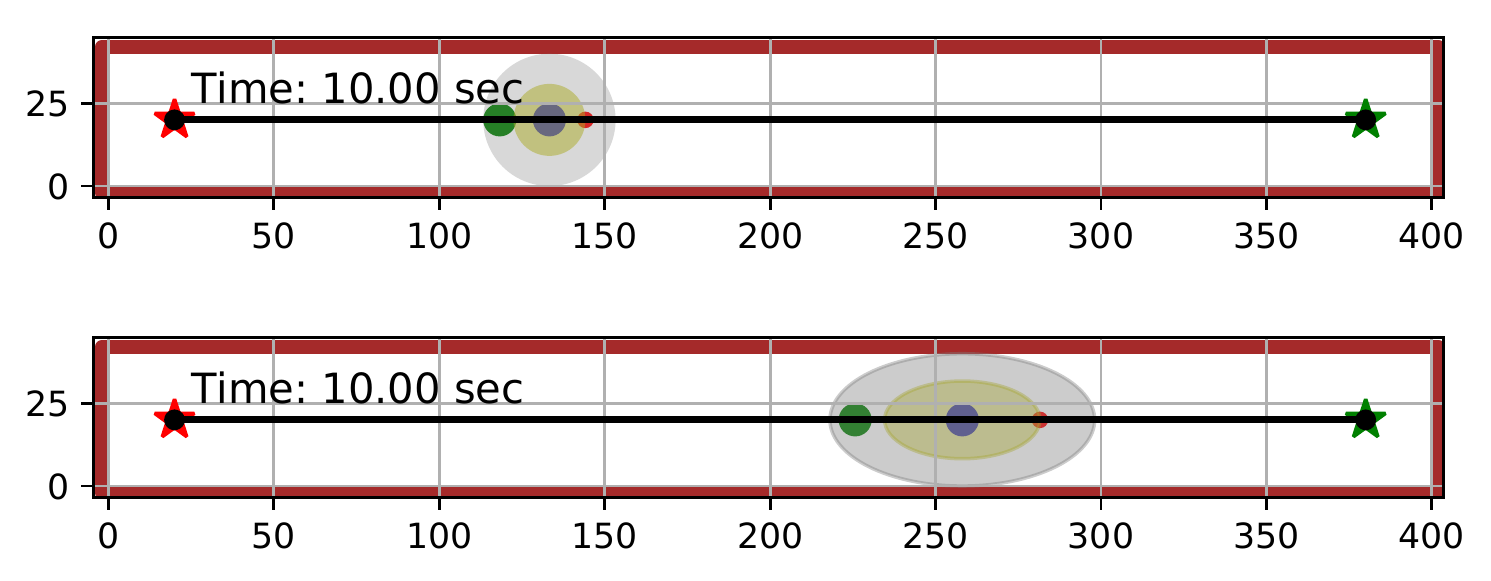}
	\caption{Comparison of controller 1 (top) and 2 (bottom) in a corridor simulation. A snapshot is shown at the same time instant for both controllers. The local energy zone (yellow) resulting from the proposed SDDM trajectory bounds fits the corridor environment well, leading to fast, yet safe, movement.\label{fig:corridor_snapshot}}%
\end{figure}

\begin{figure*}[t]
	\centering
	\includegraphics[width=0.24\linewidth]{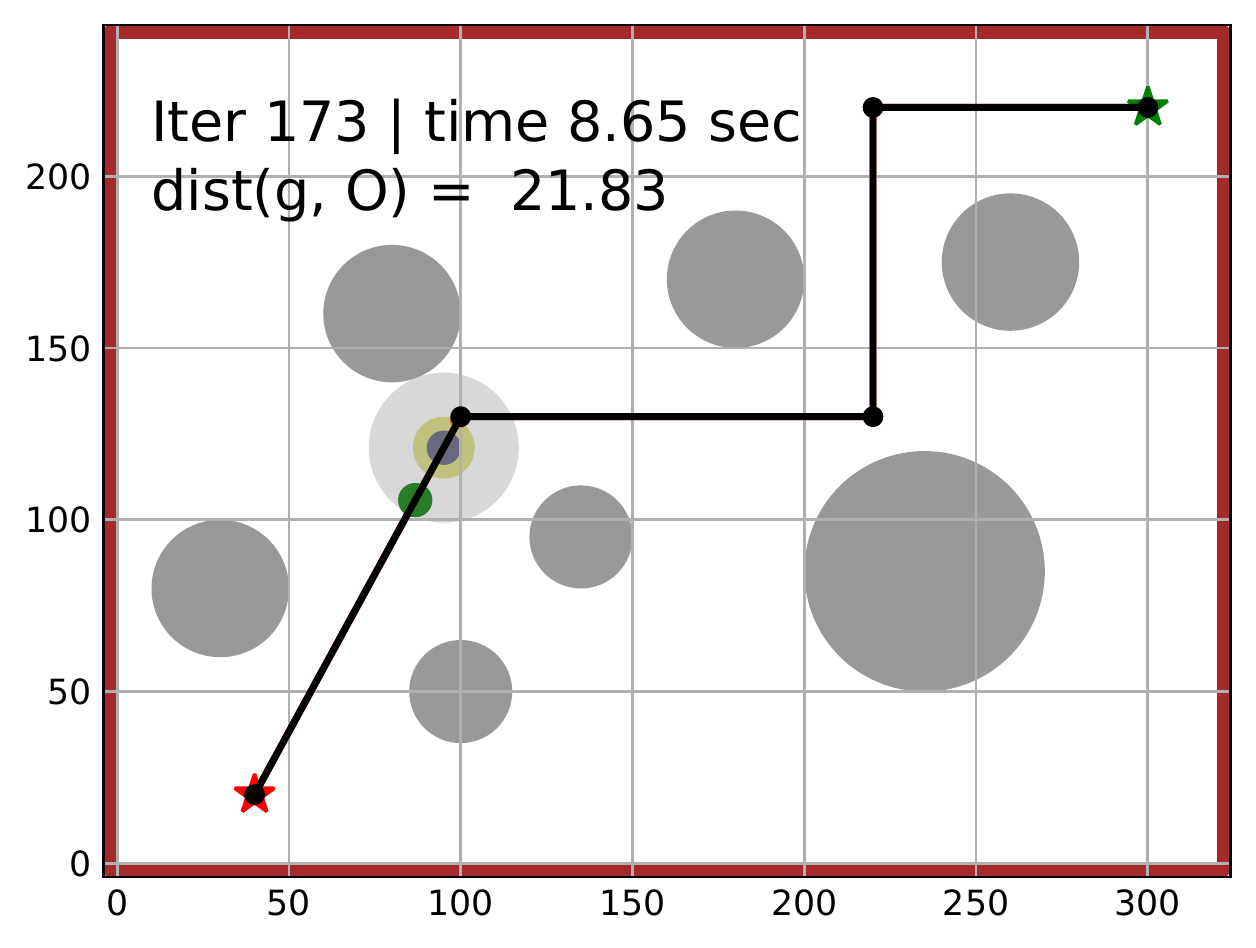}%
	\hfill%
	\includegraphics[width=0.24\linewidth]{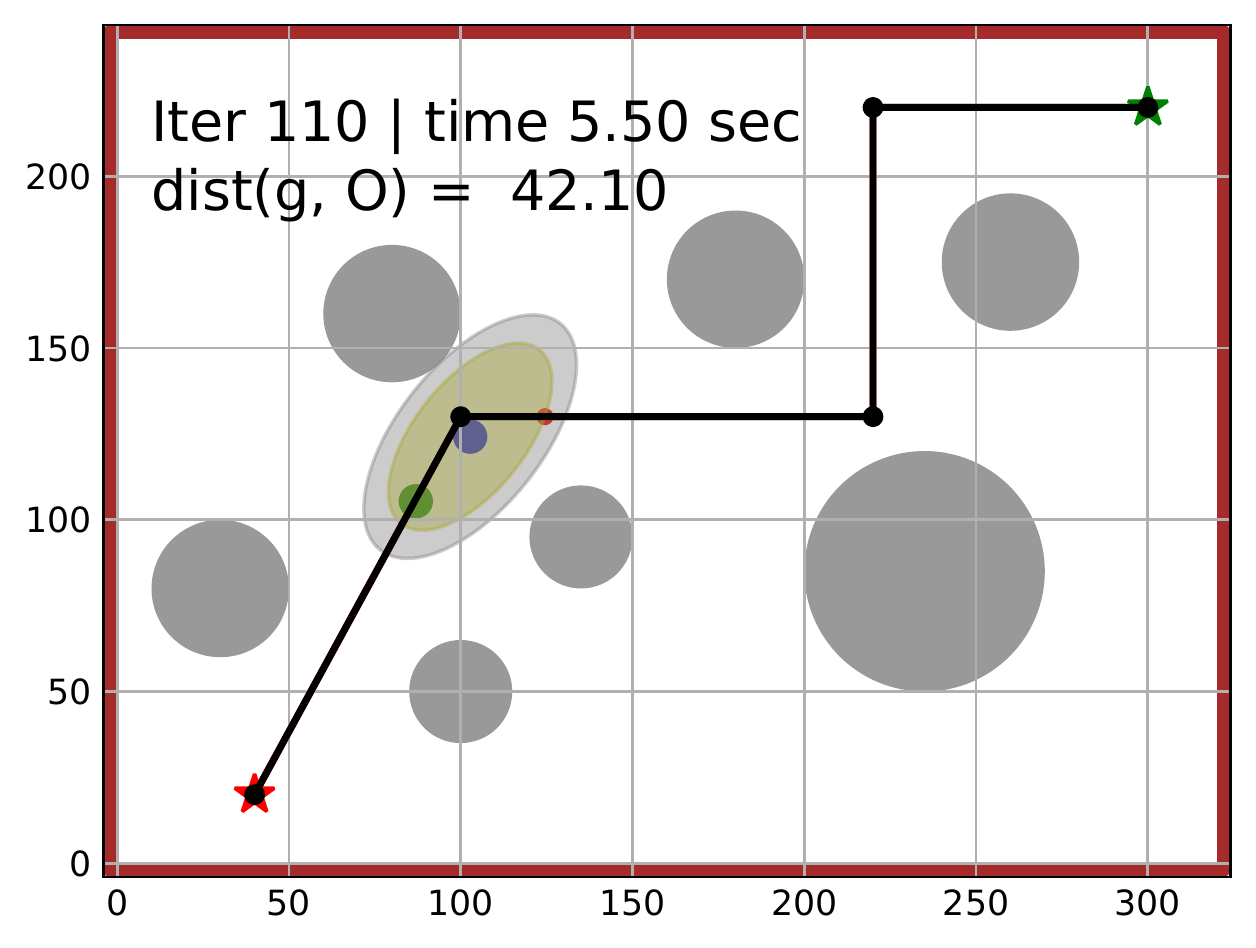}
	\hfill%
	\includegraphics[width=0.24\linewidth]{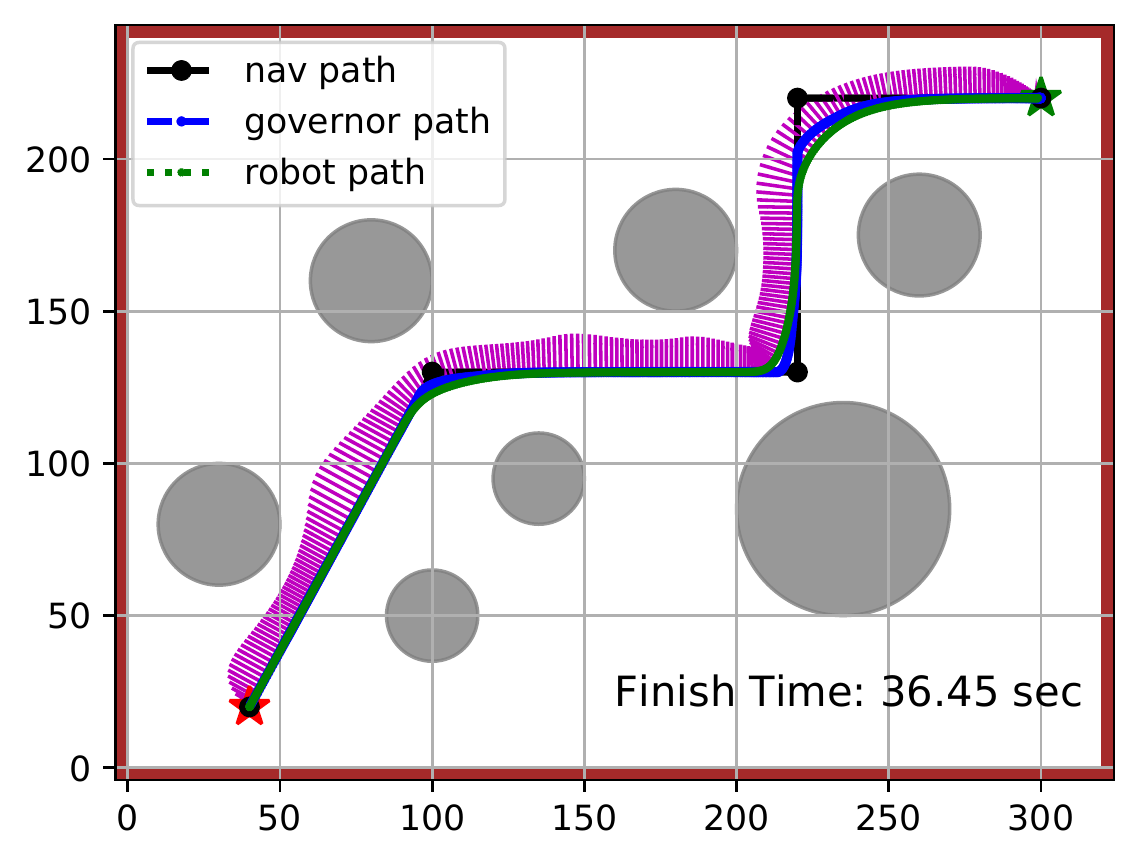}%
	\hfill%
	\includegraphics[width=0.24\linewidth]{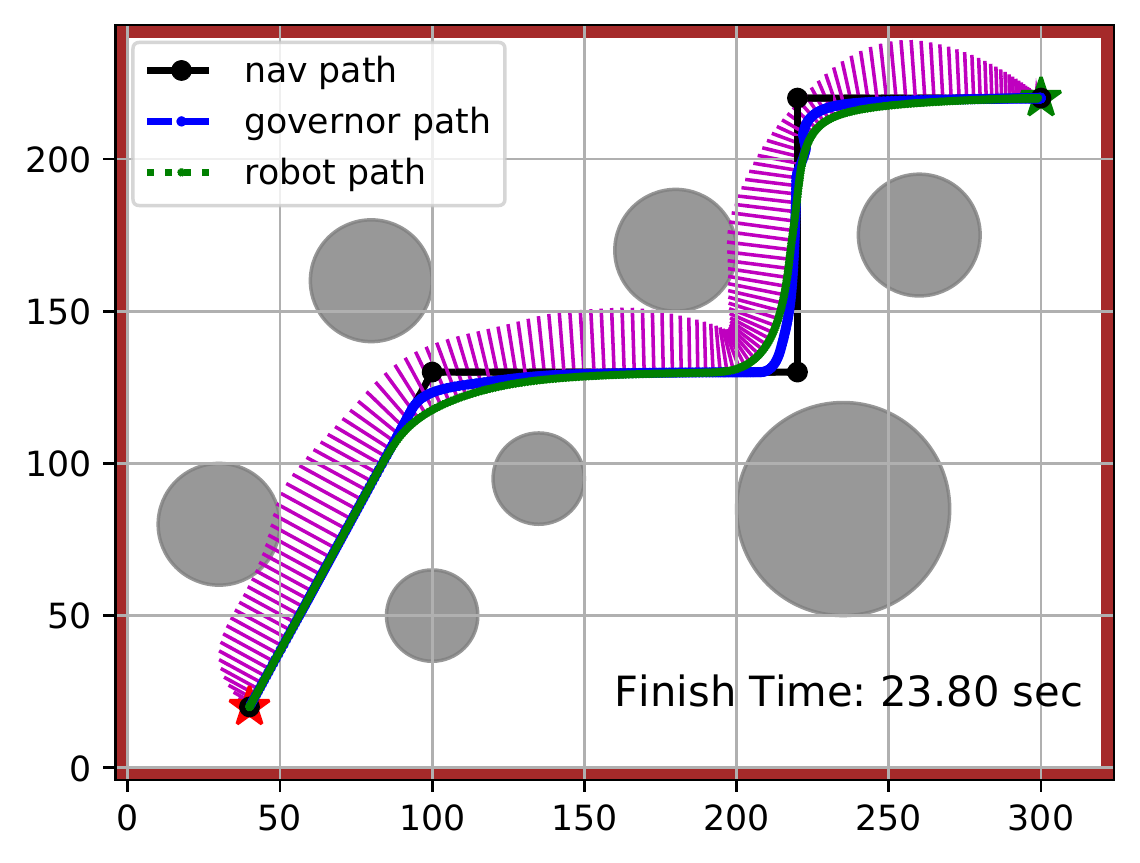}
	\caption{Simulation of the robot-governor system tracking a piecewise-linear path (black) in an environment with circular obstacles (dark gray circles). The start and end of the path are indicated by a red and green star, respectively. The two plots on the left show that the robot at around the same location behaves differently due to different distance measures. The controller using Euclidean distance is overly cautious with respect to lateral obstacles resulting in conservative motion. The system employing SDDM trajectory bounds has a larger local safe zone, which helps the robot turn fast and smoothly. The two plots on the right show the trajectories followed by the systems employing the two controllers. The velocity profiles are shown as magenta arrows perpendicular to the robot path. The controller based on SDDM trajectory bounds (rightmost) results in higher velocities compared to the controller using Euclidean ball invariant sets. Note that the path followed by the robot (green line) is also smoother, especially when turning, for the directional controller despite the higher velocity.}
	\label{fig:static_sparse_sim}
\end{figure*}

\begin{figure}[t]
	\centering
	\includegraphics[width=1\linewidth,valign=t]{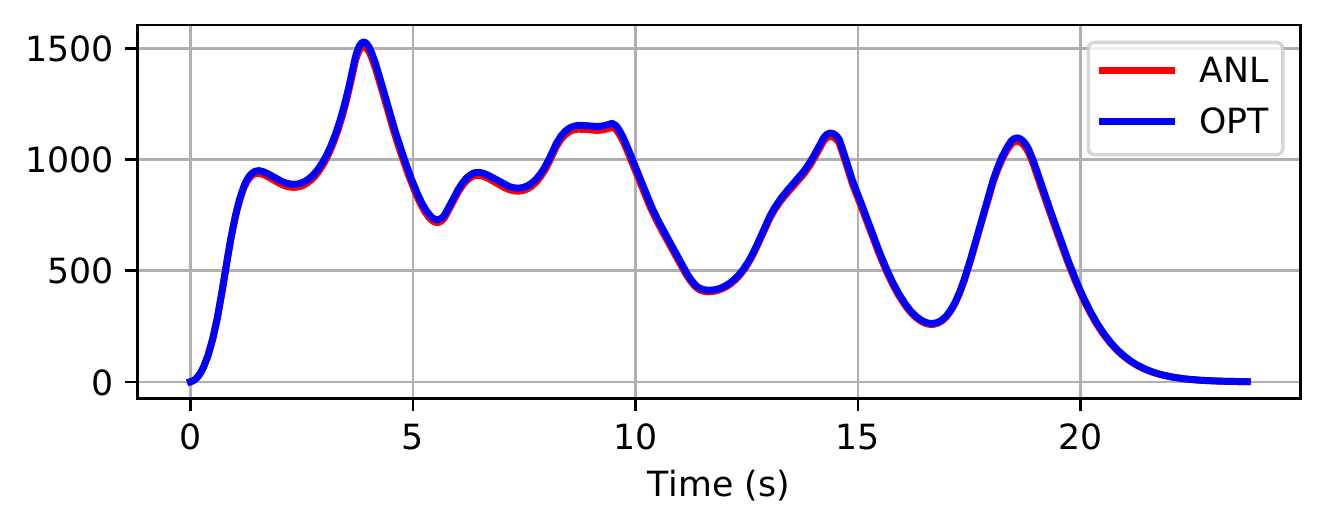}
	\vspace{-2ex}
	\caption{Output peak $\norm{\rp(t)}^2_{\BQ(t)}$ from the trajectory followed in Fig.~\ref{fig:static_sparse_sim}. The red curve is $\eta(t, \BQ(t))$ obtained analytically from eq.~\eqref{eq:critical_points}. The blue curve is $\delta(t)$ computed from the SDP optimization in eq.~\eqref{eq:SDP_UB}.  It is clear that $\delta(t)$ is an upper bound for $\eta(t, \BQ(t))$, and the bound is tight at certain moments. Analytical bounds are used in simulation, while optimization bounds are computed for comparison purpose. }
	\label{fig:Eta_max_comparison}
\end{figure}

\begin{figure}[t]
	\centering
	\includegraphics[width=\linewidth]{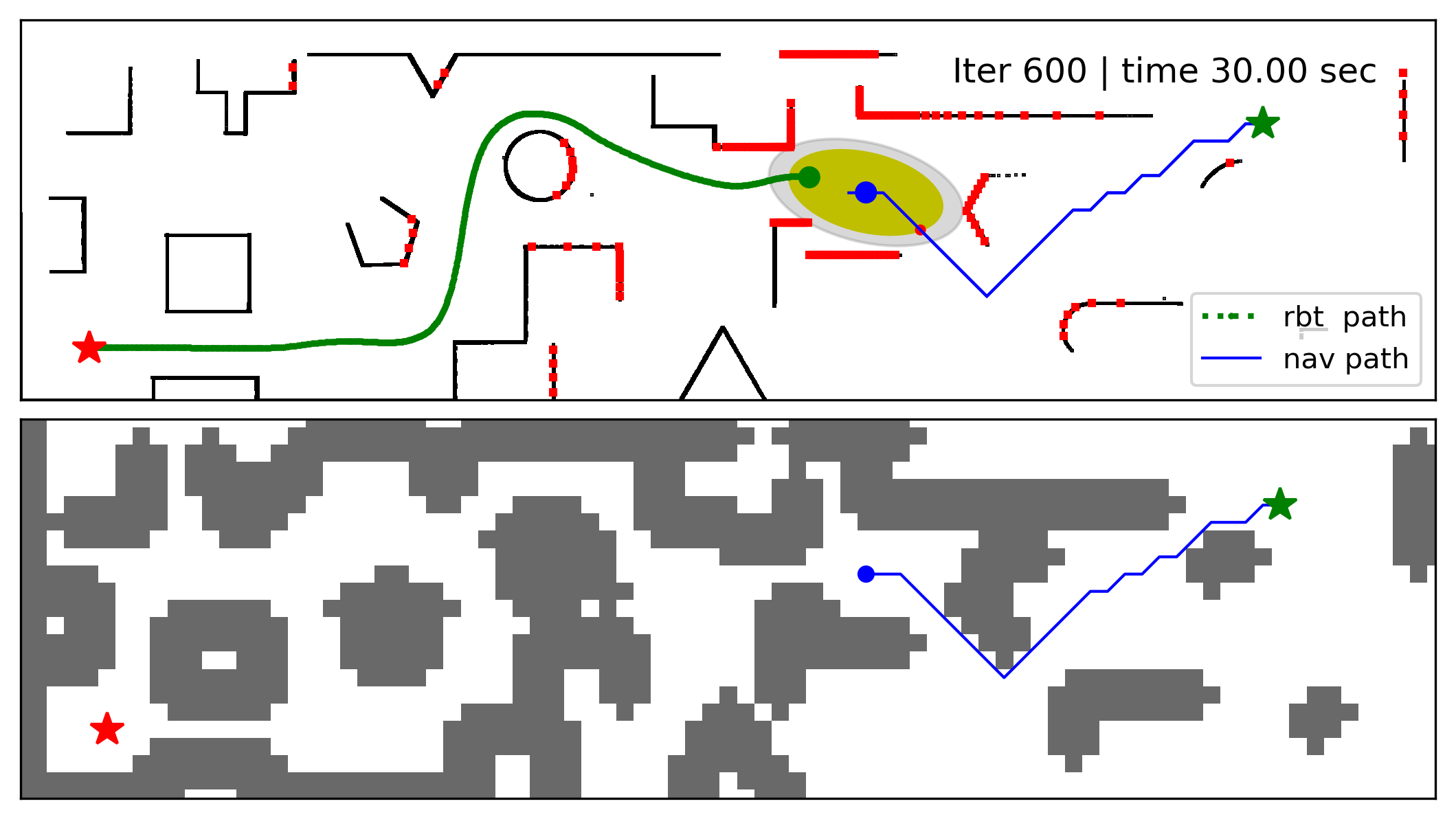}%
	\caption{Snapshot of the robot-governor system navigating the environment shown in Fig.~\ref{fig:complex_cooridor_sim}. Streaming lidar scan measurements (red dots) are used to update an occupancy grid map (black lines and white regions in top plot) of the unknown environment. An acceleration-controlled robot (green dot) follows a virtual governor (blue dot) whose motion is modulated based on the local energy zone (yellow ellipse) and the directional distance to obstacles (gray ellipse). A navigation path (blue line) is periodically replanned using an $A^*$ planner and an inflated occupancy grid map (bottom plot).}
	\label{fig:complex_cooridor_sim_2D_status}
\end{figure}

\section{Evaluation}
\label{sec:evaluation}
Consider an acceleration-controlled robot, stabilized by a proportional-derivative (PD) controller:
\begin{equation}
\label{eq:double_integrator}
\ddrp = \Bu \coloneqq -2 k \rp - \zeta \drp.
\end{equation}
The closed-loop robot-governor system is:
\begin{equation}
\label{eq:closed-loop_rgs_double_integrator}
\dot{\hat{\Bs}}  = 
\begin{bmatrix}
\drp \\
\ddrp\\
\dgp
\end{bmatrix} = 
\begin{bmatrix}
	\rv \\
	- 2 k (\rp - \gp) - \damping \rv \\
	-k_g (\gp - \lpg(\hat{\Bs}))
\end{bmatrix}
\end{equation}
Several experiments will be shown to compare the performance of two path-following controllers:
\begin{itemize}
    \item \textbf{Controller~1}~\cite{Gov_ICRA17}: uses a Lyapunov function to ensure stability and collision avoidance. The robot's kinetic and potential energy, i.e., $E(\hat{\Bs}) \coloneqq k \norm{\rp - \gp}^2 + \frac{1}{2} \norm{\rv}^2$ is used to define a spherical local safe zone.
    \item \textbf{Controller~2} (ours): uses SDDM trajectory bounds (Lemma~\ref{lemma:traj_bound}) to define an ellipsoidal local safe zone. 
\end{itemize}
In visualizations, the governor and robot positions are shown by a blue and green dot, respectively. A light-gray ellipse/ball indicates the distance from the governor to the nearest obstacles, while the local safe zone $\LS(\hat{\Bs})$ is indicated by a yellow ellipse/ball. The projected goal $\lpg$ is shown as a small red dot at the boundary of $\LS(\hat{\Bs})$. The controller parameters were $k = k_g = 1$, $\damping = 2 \sqrt{2}$, $c_1 = 1, c_2 = 4$.

\textbf{Trajectory Prediction using Different Metrics.\;} First, we demonstrate that predicting the robot trajectory using our directional metric has some desirable properties for enforcing safety constraints. Fig.~\ref{fig:traj_est_comp} compares trajectory bounds obtained from Lemma~\ref{lemma:traj_bound} for~\eqref{eq:double_integrator} using a Euclidean metric and an SDDM. Since the robot dynamics are simple, a tight directional trajectory bound $\CE_{\BQ_t}(\Bzero, \eta(t))$ can be obtained from an exact computation of the critical points according to~\eqref{eq:critical_points}. It is clear that the ellipsoid bounds on the system trajectory are less conservative (smaller area/volume) than the spherical bounds at beginning. 
Unlike a Lyapunov function, the ellipse $\CE_{\BQ(t)}(\Bzero, \eta(t))$ bounding the robot trajectory is not forward invariant. It can be shown that requiring invariance of directional ellipsoids ($\CE_{\BQ_{(t_1)}} \subset \CE_{\BQ_{(t_2)}} \,\, \forall t_2 \geq t_1$) would need infinite damping unless $\BQ = k \BI$ for some $k > 0$, causing the metric to lose directionality. In contrast to control designs based on Lyapunov function invariance, we make an interesting observation that a safe and stable controller can be defined even if the sets containing the system trajectory are not strictly shrinking over time.

\textbf{Corridor Environment.\;} We show that utilizing a directional metric in the control design alleviates the corridor effect discussed in Sec.~\ref{sec:introduction}. We setup a simulation requiring a robot to navigate through a corridor (Fig.~\ref{fig:corridor_snapshot}). The results show that controller 1, using a Lyapunov function with spherical level sets, suffers from the corridor effect while the proposed controller 2, making directional predictions about the system trajectory, does not.

\textbf{Sparse Environment with Circular Obstacles.\;} This experiment compares the two controllers in a longer path-following task in an environment with circular obstacles. Snapshots illustrating how the two controllers judge distances to obstacles and define a local energy zone are shown in Fig.~\ref{fig:static_sparse_sim}. It can be seen that the controller equipped with a directional sensing ability has a better understanding of the local environment geometry, leading to a larger, elongated local safe zone set. As a result, controller 2 does not need to slow down for low-risk lateral obstacles, leading to smoother and faster navigation. The directional bounds on the robot trajectory obtained analytically, according to eq.~\eqref{eq:critical_points}, and from the SDP in eq.~\eqref{eq:SDP_UB} are compared in Fig.~\ref{fig:Eta_max_comparison}.

\textbf{Unknown Environment with Arbitrary Obstacles.\;} This experiment demonstrates that our controller can work in a complex unknown environment, shown in Fig.~\ref{fig:complex_cooridor_sim}, relying only on local onboard measurements. The directional distance $d_{\BQ(t)}^2(\gp(t), \mathcal{O})$ from the governor to the obstacles is computed from the latest lidar scans. The path $\CP$ is re-planned from the current governor position to the goal using an occupancy grid map constructed from the lidar scans over time, as illustrated in Fig.~\ref{fig:complex_cooridor_sim_2D_status}. 

\section{Conclusion}
\label{sec:conclusion}
This paper presented a path-following controller relying on a state-dependent directional metric for trajectory prediction and safety quantification. The controller achieves fast tracking in unknown complex environments, mitigating the corridor effect, while providing safety and stability guarantees. The approach offers a promising direction for ensuring the safety of mobile autonomous systems operating in dynamically changing environments. Our design places very minimal requirements on the navigation path (piecewise-continuity) but the overall system performance depends on the path quality. Future work will focus on incorporating safety and stability considerations in path planning, applying our results to complex robot dynamics, and demonstrating the effectiveness of our design in hardware experiments.
\clearpage
\newpage

{\small
\bibliographystyle{cls/IEEEtran}
\bibliography{bib/ref.bib}
}
\end{document}